\documentclass[authordate,authblk]{miri-tech-article}

\usepackage{miritools}
  \usetikzlibrary{automata}
  \usepackage{color}

\def\Lang{\ensuremath{\mathcal L}\xspace}
\def\T{\ensuremath{T}\xspace}
\def\U{U}
\SetAlFnt{\small}
\SetKwData{TotalT}{$T^*$}
\SetKwFunction{Con}{Con}
\SetKwFunction{Claims}{Claims}
\newcommand{\ClaimsN}[2][n]{\ensuremath{\Claims_{#1}\ifstrempty{#2}{}{\!\left(\text{#2}\right)}}}
\newcommand{\ConN}[2][n]{\ensuremath{\Con_{#1}\ifstrempty{#2}{}{\!\left(\text{#2}\right)}}}
\newtheorem{thm}{Theorem}
\newtheorem{lem}{Lemma}

\title{Inductive Coherence}

\author[1,2]{Scott Garrabrant}
\author[1]{Benya Fallenstein}
\author[1,3]{Abram Demski}
\author[1]{Nate Soares}
\affil[1]{Machine Intelligence Research Institute}
\affil[2]{University of California, Los Angeles}
\affil[3]{University of Southern California}

\begin{document}

\maketitle

\begin{abstract}
While probability theory is normally applied to external environments, there has been some recent interest in probabilistic modeling of the outputs of computations that are too expensive to run. Since mathematical logic is a powerful tool for reasoning about computer programs, we consider this problem from the perspective of integrating probability and logic. Recent work on assigning probabilities to mathematical statements has used the concept of \emph{coherent} distributions, which satisfy logical constraints such as the probability of a sentence and its negation summing to one. Although there are algorithms which converge to a coherent probability distribution in the limit, this yields only weak guarantees about finite approximations of these distributions. In our setting, this is a significant limitation: Coherent distributions assign probability one to all statements provable in a specific logical theory, such as Peano Arithmetic, which can prove what the output of any terminating computation is; thus, a coherent distribution must assign probability one to the output of any terminating computation. To model uncertainty about computations, we propose to work with \emph{approximations} to coherent distributions. We introduce \emph{inductive coherence}, a strengthening of coherence that provides appropriate constraints on finite approximations, and propose an algorithm which satisfies this criterion.
\end{abstract}

\section{Introduction}

Recently there has been some interest in the problem of assigning probabilities to the outputs of computations which are too expensive to run. For example, \citet{Hennig:2015:Probabilistic} call for the development of probabilistic numerical methods that are uncertain about their calculations; \citet{Hay:2011} study metareasoning procedures for controlling Monte Carlo simulations by estimating which simulations are likely to be effective; and \citet{Rainforth:2015} give methods for estimating probabilistic program variables via Bayesian optimization.

Formal logic is a tool that is particularly well-suited for making claims about computations, such as claims of the form ``this computation will halt and produce a number larger than~7" or ``this computation will use less memory than that computation" or ``this operating system's scheduler will not deadlock" \citep{Owre:1992,Klein:2009}.

When developing methods for handling uncertainty about the results of computations, then, it is natural to approach the problem from the angle of combining logic with probability theory, a topic which has received a lot of attention; see~\citet{Russell:2015b}. Since we are using logic to reason about mathematical facts, as opposed to reasoning about an uncertain external world, the approach taken by \citet{Gaifman:1964}, \citet{Demski:2012a}, \citet{Hutter:2013:unify}, and others is particularly relevant. This approach involves assigning probabilities to logical sentences in a formal theory powerful enough to represent claims about computations, such as Peano Arithmetic (\PA) or Zermelo-Fraenkel set theory (\ZFC). 

It is standard to study distributions $P$ of this form which are \emph{coherent}, in that they obey some simple logical constraints such as $P(\bot)=0$ and $P(\phi) + P(\lnot \phi) = 1$. Coherence can be thought of as a generalization of the notion of ``consistency" to probability distributions over sentences in logic. However, coherent distributions are uncomputable---they assign probability~1 to all theorems, and thus, they must assign probability~1 to the statement ``computation $f$ outputs $x$" whenever ${f()=x}$. For this reason, coherent probability distributions cannot represent a state of uncertainty about the outputs of computations. Rather, they represent a \emph{final} state of knowledge about logical facts that a reasoner might obtain if they could think forever \citep{Cozic:2006,Halpern:2011a}.

\citet{Demski:2012a} has proposed instead investigating algorithms that \emph{computably approximate} a coherent probability distribution, that is, algorithms that output a series of probabilities for each sentence such that those probabilities converge in the limit, and such that the distribution the approximation converges to is coherent. Unfortunately, coherence in the limit is too weak for our purposes: It does not impose constraints on any individual finite approximation of the coherent distribution. For example, an approximation to a coherent distribution might assign arbitrary probabilities to some theorem $\phi$ right up until it proves it, then assigns probability~1 thereafter. If $\phi$ was a statement about a computation, this means the approximation might only start assigning reasonable probabilities to $\phi$ after actually running the computation, which defeats the purpose. To get approximations of coherent distributions that assign reasonable probabilities to $\phi$ \emph{before} running the computation, we require some notion like coherence that we can impose on \emph{approximations} to the final distribution.

In this paper, we propose one such property, \emph{inductive coherence}, argue that it is a desirable generalization of coherence to approximations of coherent distributions, and show that a an inductively coherent approximation scheme exists. Roughly speaking, inductive coherence demands that, for any pattern in what is provable that can be identified by a Turing machine in polynomial time, the approximation must recognize and exploit that pattern ``not much later" than that Turing machine. We define this more formally in \Sec{icoherence}. In \Sec{prior} we propose a modification of Demski's algorithm~\shortcite{Demski:2012a} and show that it is inductively coherent. We conclude with a discussion of open problems and future work.

\paragraph{Other Related Work}

The study of assigning probabilities to sentences in mathematical logic dates back to \citet{Los:1955,Gaifman:1964}; see \citet{Hailperin:1984} for a more thorough history. Since then, the idea has been extended to, e.g., infinitary logic \citep{Scott:1966}, databases with uncertain data \citep{Suciu:2011}, and higher-order logic \citep{Hutter:2013:unify}. Computing probability assignments for logical statements can be seen as an extension of these approaches to the case where the reasoner's beliefs may be incoherent; see also the work of \citet{Muino:2011,Potyka:2015}, who study methods for measuring and handling inconsistency in knowledge bases.

Another method for unifying logic with probability is inductive logic programming \citep{Muggleton:1994,Nienhuys:1997}. For example, \citet{DeRaedt:2008} have described techniques for learning from proofs and program traces, and \citet{Richardson:2006} propose combining first-order logic and probabilistic graphical models using a method known as ``Markov logic networks." For a review of recent work, refer to \citet{Russell:2015b}.

Methods for reasoning probabilistically about the outputs of computations are a type learning scheme for probabilistic logic programs. For more on this subject, refer to \citet{Ng:1992,Muggleton:1996,Sato:1997,Poole:1997,Ngo:1997,Koller:1997,Lukasiewicz:1998,Kersting:2000}.

\section{Inductive Coherence} \label{sec:icoherence}

In this paper we study probability distributions over sentences of first-order logic, with the goal of describing computable distributions that assign reasonable probabilities to claims about computations that are too expensive to be run. (For example, imagine a reasoner that wants to know whether a particular $O(n^2)$ computation will outperform a particular $O(n \log n)$ computation on a particular dataset, without taking the time to run both computations.) We fix a theory $\T$ that is powerful enough to make claims about computations, such as~\PA. We let $\Lang$ denote the set of sentences is the language of $\T$.

``Coherence" can be seen as the natural generalization of consistency to probability distributions. It demands that theorems be certain, and the probabilities of mutually exclusive sentences add:

\begin{definition}[Coherence]
  A probability distribution over sentences in $\Lang$ is a function $P : \mathcal{L}\to[0,1]$ from sentences to probabilities. It is called \emph{coherent} with respect to the $\Lang$-theory~$\T$ if the following three conditions hold:
  \begin{enumerate}
    \item If $\phi$ is a theorem of $\T$, $P(\phi)=1$.
    \item If $\neg \left( \phi\wedge\psi \right)$ is a theorem of $\T$, $P(\phi \vee \psi)=P(\phi)+P(\psi)$.
  \end{enumerate}
\end{definition}

\noindent It is not hard to see that coherence ensures $P$ obeys other obvious logical constraints, such as $P(\phi) = 1 - P(\lnot \phi)$ and ${P(\phi \to \psi)=1} \implies {P(\psi) \ge P(\phi)}$. We say that coherence is a generalization of consistency because $P$ agrees with $\T$ on all theorems and contradictions, but can assign probabilities to undecidable sentences so long as those probabilities obey logical constraints. \citet{Gaifman:1964} has shown that any coherent $P$ is isomorphic to a measure $\mu$ on complete consistent extensions of $\T$---in other words, a coherent $P$ assigns probability to undecidable sentences in a fashion that corresponds to choosing some weighted mixture of all possible consistent ways to assign truth values to sentences.

Coherence is a \emph{very} strong constraint. For example, if $\T$ is \PA, then $P$ must assign probability~1 to all true statements about the behavior of computations. One way to think of a coherent distribution is that it represents the state of knowledge a reasoner could achieve after thinking \emph{forever}, after proving everything provable, refuting everything refutable, and assigning consistent probabilities to everything undecidable. It seems reasonable to ask that the \emph{limit} of a good reasoner's beliefs about logical statements should be coherent, but no computable distribution can achieve coherence. This motivates the study of computable \emph{approximation schemes} to coherent distributions, that is, algorithms which output a sequence of probabilities for each sentence such that the sequences converge, and the resulting probability assignments are coherent.

\begin{definition}[Approximation Scheme]
  An \emph{approximation scheme} is a Turing machine $M$ which takes as input a natural number $n$ and an encoding of a sentence $\phi$, and outputs a rational number $M_n(\phi)\in\mathbb{Q}\cap[0,1]$.
\end{definition}

\noindent We can think of $M$ as a machine that runs for longer and longer as $n$ gets larger, producing better and better estimates of the probability of $\phi$ each iteration. We concern ourselves with approximation schemes that converge to a coherent distribution in the limit.

\begin{definition}[Approximation]
  An approximation scheme $M$ \emph{approximates} $P:\Lang\rightarrow[0,1]$ if $$\lim_{n\rightarrow\infty} M_n(\phi)=P(\phi)$$
for all $\phi$. If there exists an~$M$ which approximates $P$, then we say that $P$ is ``approximable."
\end{definition}

Coherence in the limit does not impose strong constraints on an approximation scheme. Given a sentence $\phi$ describing a claim about a computation, $M$ could simply assign it arbitrary probabilities right up until it can run the computation, and then assign it probability~1 or~0 thereafter. $M$ could then be coherent in the limit, but it would never assign reasonable non-extreme probabilities to claims about computations too expensive to run. To get an $M$ that accurately estimates the results of computations before it can run them, we need to impose stronger constraints on the approximations.

Clearly, we cannot demand that the approximate distributions $M_n$ be completely coherent. We could try weakening coherence by demanding that $M_n$ assign probability~1 to all theorems \emph{that have a proof of length $n$ or less}, and indeed, this is the approach taken by \citet{Demski:2012a}. However, in our setting, this runs into the same problem as above: $M$ could still assign arbitrary probabilities to a decidable $\phi$ right up until it proves $\phi$ one way or the other, and might not have anything reasonable to say about the behavior of computations before running them. We require some other weakening of coherence that ensures $M$ places reasonable probabilities on $\phi$ \emph{before} it can run the computation.

Intuitively, we want $M$ to start noticing \emph{patterns} in computations before it's able to actually run them. For example, consider some computation $f$ that takes an input and either outputs 1, outputs 0, or loops. Let $\phi_n$ be the claim $f(n)=0$ and $\psi_n$ be the claim $f(n)=1$. We want $M$ to start assigning probabilities to $\phi_n$ and $\psi_n$ that sum to at most~1, and we want it to start doing so before it can simply compute $f(n)$.

More generally, we want $M$ to recognize patterns such as ``these claims are mutually exclusive" and ``these claims are equivalent." We formalize this idea as follows. Take any method for quickly (in polynomial time) identifying patterns of this form. A good approximation scheme should eventually be able to recognize and exploit that pattern, and ``not much slower" than the polynomial time method, even if the size of the shortest proof that the pattern holds grows superexponentially.

\begin{definition}[Quickly Computable]
  A sequence of sentences $\{\phi_n\}$ is called \emph{quickly computable} if there exists a Turing machine that outputs all the sentences in the sequence in order, and outputs $\phi_n$ by a time polynomial in $n$.
\end{definition}

\begin{definition}[Inductive Coherence] \label{def:ic}
  An approximation scheme $M$ is called \emph{inductively coherent} with respect to \T if it satisfies the following three conditions:
    \begin{enumerate}
    \item $\lim_{n\rightarrow\infty} M_n(\bot)=0$
    \item $\lim_{n\rightarrow\infty} M_n(\phi_n)$ converges whenever $\{\phi_n\}$ is quickly computable and $\phi_n\rightarrow \phi_{n+1}$ is provable in \T for all $n$.
    \item $\lim_{n\rightarrow\infty} M_n(\phi_n)+M_n(\psi_n)+M_n(\chi_n)=1$ whenever $\{\phi_n\}$, $\{\psi_n\}$, and $\{\chi_n\}$ are quickly computable, and for all $n$, it is provable in \T that $\{\phi_n,\psi_n,\chi_n\}$ is a partition of truth (meaning exactly one of them is true).
  \end{enumerate}
\end{definition}

\noindent To gain an intuition for how \Def{ic} guarantees that an inductively coherent $M$ recognizes patterns in quickly computable sequences ``not much slower" than the Turing machine $q$ that quickly computes them, imagine that $q(n)$  outputs pairs $(\phi_n, \psi_n)$ in time polynomial in $n$ such that each $\phi_n$ is provably equivalent to $\psi_n$, but the shortest proof of equivalence grows superexponentially fast in $n$. We want $M$ to eventually, on some iteration $n$ and thereafter, ``recognize the pattern" and start assigning roughly the same probabilities to $\phi_n$ and $\psi_n$. Note that this is a claim about the iteration $n$ by which $M$ must start recognizing the pattern identified by $q$, not a claim about the runtime of $M_n(\phi)$ on an individual $\phi$, which may be exponential or worse.

\begin{thm} \label{thm:equiv}
  If $M$ is inductively coherent, and $\{ \phi_n \}$ and $\{ \psi_n \}$ are quickly computable, and $\phi_n$ is provably equivalent to $\psi_n$ for all $n$, then $$\lim_{n \rightarrow \infty} M_n(\phi_n) - M_n(\psi_n) = 0.$$
\end{thm}
\begin{proof} If $\{\phi_n\}$ is quickly computable then so is $\{\lnot \phi_n\}$. The constant sequence $\{\bot\}$ is quickly computable. Apply property $3$ to the sequences $\{\phi_n\}$, $\{\lnot \phi_n\}$, and $\{\bot\}$, and then to $\{\psi_n\}$, $\{\lnot \phi_n\}$, $\{\bot\}$. Subtracting the results, we have $$\lim_{n\to\infty} M_n(\phi_n) + M_n(\lnot \phi_n) - M_n(\psi_n) - M_n(\lnot \phi_n) = 0,$$ because $\lim_{n\to\infty} M_n(\bot)=0$.
\end{proof}

Provable equivalence is only one type of pattern that an inductively coherent $M$ exploits before it can find the associated proofs. As a second example, if there is any quickly computable method for identifying sentences that are provable (even if the proofs are very long), then $M$ must recognize those patterns as well.

\begin{thm} \label{thm:converge}
  If $M$ is inductively coherent, for any quickly computable sequence $\{ \phi_n \}$ of theorems of $\T$ we have $$\lim_{n \rightarrow \infty} M_n(\phi_n) = 1.$$
\end{thm}
\begin{proof}
  Apply property 3 to the sequences $\{ \phi_n \}$, $\{ \bot \}$, $\{ \bot \}$. Since each $\phi_n$ is provable, we have partitions of truth as desired, and $\lim_{n\rightarrow\infty} M_n(\bot)=0$.
\end{proof}

\noindent This implies that the criterion of inductive coherence captures some of what it means to reason well about computations that are too expensive to run. Recall that statements making true claims about the output of a computation are theorems, because the statement can be proven by providing an execution trace of the computation. \Thm{converge} implies that if there is a polynomial-time method for generating true statements about computations, then after some number of iterations $n$, $M_n$ incorporates that pattern into its probability assignments.

Property 1 of \Def{ic} is fairly trivial. Property 2 implies that $M_n(\phi_n)$ converges if $\{\phi_n\}$ is quickly computable and all the $\phi_n$ are provably equivalent; the more general form of Property 2 is used in \Thm{exclusive}. Property 3 is the meat of inductive coherence; it ensures $M$ recognizes exclusivity relationships between provable sentences. While \Def{ic} only mentions partitions of truth of size 3, it is strong enough to guarantee $M$ recognizes arbitrarily large partitions of truth.
\begin{thm} \label{thm:converge2}
  If $M$ is inductively coherent, for any list of $k$ quickly computable sequences of sentences, $\{ \phi^1_n \}, \ldots \{ \phi^k_n \}$, such that for each $n$, it's provable that $\phi^i_n$ is true for exactly one $i$, we have that $$\lim_{n\rightarrow\infty} \sum_{i=1}^k M_n(\phi^i_n)=1.$$
\end{thm}
\begin{proof}
  The proof works by induction on $k$. For $k=1$, this is Theorem \ref{thm:converge}. For $k=2$, apply property 3 to the sequences $\{ \phi^1_n \}$, $\{ \phi^2_n \}$, $\{ \bot \}$. For $k=3,$ this is exactly property 3.

  For $k>3$, consider the list $\{ \phi^1_n \}, \ldots \{ \phi^{k-2}_n \}, \{\phi^{k-1}_n\vee \phi^{k}_n\}$; the list $\{\phi^{k-1}_n\},\{\phi^{k}_n\},\{\neg (\phi^{k-1}_n\vee \phi^{k}_n)\}$; and the list $\{\phi^{k-1}_n\vee \phi^{k}_n\},\{\neg( \phi^{k-1}_n\vee \phi^{k}_n)\}$.

These lists are of length $k-1$, $3$, and $2$ respectively and all satisfy the conditions of this theorem. Apply this theorem to all three lists, add the first two results and subtract the third. This gives $\lim_{n\rightarrow\infty} \sum_{i=1}^k M_n(\phi^i_n)=1+1-1=1.$
\end{proof}

Intuitively, an inductively coherent $M$ is good at identifying any pattern in what is provable that can be expressed using the properties of \Def{ic}. As an example, observe that if $q$ can quickly compute an infinite sequence of provably mutually exclusive sentences, $M$ must eventually start noticing that those sentences are mutually exclusive, not much later than $q$:

\begin{thm} \label{thm:exclusive}
If $M$ is inductively coherent, then for a quickly computable sequence of mutually exclusive sentences, $\{\phi_n \}$, we have $\lim_{n \rightarrow \infty} M_n(\phi_n) = 0$.
\end{thm}
\begin{proof}
Define $\psi_n$ to be the disjunction of all $\phi_{i \leq n}$. Applying property 2 to $\{ \psi_n \}$, we have that $\lim_{n \rightarrow \infty} M_n(\psi_n)$ converges to some $p$. Applying property 3 to the sequences $\{ \psi_n \}$, $\{ \neg \psi_n \}$, and $\{ \bot \}$, we have that $\lim_{n \rightarrow \infty} M_n(\lnot \psi_n)$ converges to $1-p$. Therefore, applying property 3 to $\{ \psi_{n-1} \}$, $\{ \phi_n \}$, and $\{ \neg \psi_n \}$ shows that $\lim_{n \rightarrow \infty} M_n(\phi_n) = 0$ as desired. (Note that $\{ \psi_{n} \}$, $\{ \neg \psi_n \}$, and $\{ \psi_{n-1} \}$ are all quickly computable if $\{ \phi_n \}$ is.)
\end{proof}

\noindent Does an inductively coherent $M$ quickly identify \emph{all} quickly identifiable patterns in claims about computations? Probably not; limitations are discussed in \Sec{discussion}.  However, we can show that $M$ will \emph{eventually} identify all patterns in which sentences are provable, by showing that an inductively coherent $M$ is coherent in the limit. This may be surprising at first glance, given that \Def{ic} only mentions convergence for sequences that can be computed in polynomial time. The trick is that any constant sequence $\phi_n \coloneqq \phi$ is ``quickly computable," by the Turing machine that ignores $n$ and always outputs $\phi$.

\begin{thm}
  If $M$ is inductively coherent, then $$P(\phi):=\lim_{n\rightarrow\infty} M_n(\phi)$$ is well-defined, approximable, and coherent.
\end{thm}

\begin{proof}
  That $P(\phi)$ is well-defined follows from property 2 and the fact that the constant sequence $\{ \phi \}$ is quickly computable for any $\phi$. Approximability then follows trivially from the definition of $P(\phi)$. $P(\phi)$ is in $[0,1]$ because $M_n(\phi)$ is in $[0,1]$, and the limit of any sequence in $[0,1]$ is in $[0,1]$. It remains to show that $P$ is coherent.

  The first property of coherence follows from properties 1 and 3 with the partition $\{ \phi, \bot, \bot \}$. The second follows with the partition $\{\phi, \psi, \neg \phi \wedge \neg \psi \}$.
\end{proof}

Thus we see that we are justified in saying that an inductively coherent $M$ must both quickly identify some patterns in what is provable, and eventually identify all patterns. This implies an inductively coherent $M$ would assign reasonable probabilities to sentences describing the behavior of computations, even before the computation can be run, because statements about computations are decidable. If there is \emph{any} polynomial-time method for accurately noticing relationships between computations, $M$ will eventually recognize it and distribute its probability mass accordingly.

One way to look at inductive coherence is this: Coherence in the limit requires that each \emph{individual sentence} is eventually assigned a reasonable probability, possibly only after that sentence is decided. Inductive coherence requires that there is some uniform bound past which \emph{all} theorems that can be quickly identified as theorems start to be assigned high probability, as quickly as they can be identified.

\section{An Inductively Coherent Approximation Scheme} \label{sec:prior}

We now turn our attention to providing an inductively coherent approximation scheme $M^*$ which approximates a coherent distribution $P^*$. We do this by defining a variant of the distribution and approximation scheme defined by \citet{Demski:2012a}, and showing that our variant is inductively coherent.

We begin by defining an uncomputable process (which we later show to be approximable) that builds a random complete, consistent extension \TotalT of \T. We define $P^*(\phi)$ to be the probability that $\phi \in \TotalT$ if \TotalT is sampled according to this process. To build \TotalT, we sample random Turing machines according to a simplicity prior, and interpret their outputs as claims about which sentences are true. If the machine makes claims that are consistent with \TotalT so far, we add those claims to \TotalT and repeat. Otherwise we discard that machine and repeat. Continuing indefinitely, \TotalT will (with probability~1) be a complete consistent extension of \T in the limit.

To formalize this idea, fix a universal Turing machine $\U$ with an advance-only output tape, with its input tape initialized to a random infinite bitstring~$b$. We can interpret an infinite bitstring as a self-delimiting encoding of a Turing machine followed by a specification of the initial state of that machine's tape. Fix an enumeration of sentences in the language $\Lang$, and interpret the output of $\U$ on input tape~$b$ as a sequence of sentences. Write $\U(b)$ for the set of sentences output by $\U$ on the input~$b$, and write $\U_t(b)$ for the finite set of sentences output during the first~$t$ steps of operation. For a finite bitstring $x$, write $\U(x)=S$ if $\U(b)=S$ for all~$b$ which have~$x$ as a prefix, and similarly for $\U_t(x)$. Note that if the length of~$x$ is at least~$t$, then $\U_t(x)$ is always well-defined, because $\U$ cannot read more than~$t$ input bits in~$t$ steps. The (uncomputable) process converging on \TotalT is then defined by \Alg{Complete}.

\begin{algorithm}[ht]
  \KwIn{$b_1$, $b_2$, $\ldots$\,, an infinite list of infinite bitstrings.}
  \KwData{$\TotalT \leftarrow T$, the extension of \T under construction.}
  \For{$b$ in $b_1$, $b_2$, $\ldots$}{
    \If{$\TotalT \cup \U(b)$ is consistent}{
      $\TotalT \leftarrow \TotalT \cup \U(b)$\;
    }
  }
  \caption{A method for constructing a complete, consistent extension of $T$.\label{alg:Complete}}
\end{algorithm}

We define $P^*(\phi)$ to be the probability that $\phi \in \TotalT$ when the $b_i$ are chosen uniformly at random (e.g., by fair coin tosses; recall that a single stream of coin tosses can encode an infinite sequence of infinite bitstrings).

$P^*$ has the desirable property that, for every noncontradictory sentence $\phi$, $P^*(\phi)$ is lower-bounded by the complexity of the Turing machine that outputs only $\phi$. To see this, let $w_\phi$ be the bitstring encoding that machine with respect to $U$; the chance that $b_1$ starts with $w_\phi$ is at least $2^{-|w_\phi|}$.

To see that $P^*(\phi)$ is coherent, note that with probability~1 $\TotalT$ is a complete consistent extension of $\T$, so $P^*$ is isomorphic to a distribution~$\mu$ over complete consistent extensions of $\T$, which means it is coherent~\citep{Gaifman:1964}. $P^*$ is uncomputable, but can be approximated by \Alg{mstar}.

\begin{algorithm}
  \SetKwData{Theory}{$\Phi$}
  \BlankLine
  \Fn{\ClaimsN{$b_1$, $\ldots$\,, $b_{2^n}$}}{
    $\Theory \leftarrow \text{the first $n$ axioms of $\T$}$\;
    \For{$i$ in $0 \ldots 2^n$}{
      $S \leftarrow \text{$\U_{2^n}(b_i)$ interpreted as a list of sentences}$\;
      \If{\ConN{$\Theory \cup S$}}{
        $\Theory \leftarrow \Theory \cup S$\;
      }
    }
    \For{$\phi$ in \Theory}{\Output{$\phi$}}
  }
  \BlankLine
  \Fn{\ConN{\Theory}}{
    \If{a proof of length $\le 2^n$ proves $\Theory$ inconsistent}{\KwRet{false}}
    \For{$S \subset \Theory$}{
      \For{$\phi$ of length $\le 2^n$ such that $S \cap \{\phi, \lnot \phi\}$ is empty}{
        \If{$\lnot$\ConN{$S \cup \{\phi\}$} and $\lnot$\ConN{$S \cup \{\lnot \phi\}$}}{
          \KwRet{false}
        }
      }
    }
    \KwRet{true}
  }

  \BlankLine
  \Fn{$M^*_n(\phi)\,$}{
    \KwRet{the probability \ClaimsN{$b_1$, $\ldots$\,, $b_{2^n}$} outputs $\phi$ when the $b_i$ are uniform random bitstrings of length $2^n$.\DontPrintSemicolon\;}
  }
  \caption{Computable approximation scheme for $P^*$ \label{alg:mstar}}
\end{algorithm}

\noindent Of note is the function \ConN{}, which checks whether a set of sentences $\Phi$ is ``consistent enough" for time $n$. It checks not only whether $\Phi$ can be proven inconsistent with a proof of length $2^n$ or less, but also whether there is a subset $S\subset\Phi$ and sentence $\phi$ such that both $\Phi\cup\{\phi\}$ and $\Phi\cup\{\neg\phi\}$ can be proven inconsistent in length $2^n$. (This implies that $\Phi$ is inconsistent, but the proof may be longer than length $2^n$ unless one of $\phi$ or $\lnot \phi$ is added.) This gives \ConN{} a convenient closure property.

\begin{thm} \label{thm:main}
$M^*$ is an inductively coherent approximation scheme which approximates $P^*$.
\end{thm}

\noindent We show that $M^*$ recognizes quickly computable theorems, which is suggestive. The rest of the proof is in \App{ic}.

\begin{lem} \label{lem:main}
If $\{\phi_n\}$ is a quickly computable sequence of theorems, then $\lim_{n\rightarrow\infty} M^*_n(\phi_n)=1.$
\end{lem}
\begin{proof}
  Write $\Phi_n$ for a random run of $\ClaimsN{$b_1$ \ldots $b_{2^n}$}$ when the $b_i$ are chosen uniformly at random. We want to show that, for $n$ large enough, $\PP(\phi_n \in \Phi_n) > 1 - \varepsilon.$  Since $\{\phi_n\}$ is quickly computable, there exists a finite bitstring $w_\phi$ such that $\U(w_{\phi})$ outputs the sentences $\{\phi_n\}$ in order, and for all sufficiently large $n$, $\phi_n\in\U_{2^n}(w_{\phi})$.

  There exists a $N_0$ such that with probability at least $1-\varepsilon/2$, at least one of $b_1 \ldots b_{2^{N_0}}$ starts with $w_\phi$, and $\phi_n\in\U_{2^n}(w_{\phi})$ for all $n\geq N_0$. There also exists a $N_1$ such that with probability $1-\varepsilon/2$, for every subset $S$ of $\{1,\ldots 2^{N_1}\},$ either $\T\cup\bigcup_{i\in S}\U(b_i)$ is consistent or $\lnot \ConN[N_1]{$\T\cup\bigcup_{i\in S}\U_{2^{N_1}}(b_i)$};$ simply choose $N_1$ large enough that any inconsistencies can be uncovered with sentences output by time $2^{N_1}$ and proofs of length less than $2^{N_1}$. (This is possible because only finitely many proofs of inconsistency are needed.)

  Choose $N_1 \ge N_0$. For all $n\geq N_1$, with probability at least $1-\varepsilon/2$, one of the sampled machines (namely $w_\phi$) outputs $\phi_i$ for all $i\leq n$. Then, with probability at least $1-\varepsilon/2$, this implies that this machine will end up contributing to $\Phi_n$ because any machine before $w_\phi$ inconsistent with $w_\phi$ (which outputs only theorems) has been discarded. Therefore, with probability at least $1-\varepsilon$, ${\phi_n\in \Phi_n}$.
\end{proof}

\section{Conclusions} \label{sec:discussion}
We have proposed \emph{inductive coherence} as a strengthening of coherence in the limit. Inductive coherence requires that computable distributions assign probabilities to claims about computations that are reasonable before they're able to run these computations. Specifically, if there is any polynomial-time method for identifying patterns in what is provable, an inductively coherent $M$ must eventually recognize and exploit that pattern, eventually assigning probabilities that are coherent with respect to that pattern. This implies that inductive coherence captures some of what we mean when we ask for a probability distribution that assigns reasonable probabilities to claims about computations.

However, an inductively coherent $M$ doesn't necessarily recognize all patterns in the behavior of computations. For example, consider: Is the $10^{100}$'th decimal digit of $\pi$ a $7$? It seems that in lieu of additional knowledge and the ability to compute the digit, a reasonable estimator should assign this event 10\% probability. Reasonable predictors of computations should be able to recognize similar patterns, such as ``this computation returns an error one time in ten," and assign probabilities accordingly.

More formally, imagine we have some sequence of deterministic computations that output a one on $\sfrac{1}{10}$ of their inputs. Imagine further that there is no polynomial-time algorithm that has better average squared error, when predicting this sequence, than the algorithm ``\Output{$\sfrac{1}{10}$}." It seems reasonable to ask that a predictor of computations start assigning probability $\sfrac{1}{10}$ to the next element in the sequence eventually, until it has enough resources to compute the actual answer. However, we have no reason to expect that an inductively coherent $M$ would have this property. \citet{Garrabrant:2015:alu} study computable distributions that can do this; it is not yet clear how to reconcile our framework with theirs.

This demonstrates that further constraints on approximation schemes are likely necessary before we can define computable distributions that are able to recognize all the patterns in the behavior of computations that humans can easily recognize. Inductive coherence gives us approximate distributions that have some desirable properties in their predictions about computations, but more work is needed before we can say we understand how to assign reasonable uncertainty to claims about computations in general.

\appendix
\section{Proof that \texorpdfstring{$M^*$}{M*} is Inductively Coherent} \label{app:ic}

Let $b_1, b_2, \ldots$ be an infinite list of infinite bitstrings, generated uniformly at random.  Let $\TotalT$ be the complete extension of \T that \Alg{Complete} converges to on the input $b_1, b_2, \ldots$ Let $\Phi_n=\ClaimsN{$b_1$, $b_2$, $\ldots$}$. Note that \ClaimsN{} only reads the first $2^n$ bits of the first $2^n$ bitstrings.

\begin{lem} \label{lem:inverse}
If $\{\phi_n\}$ is a quickly computable sequence of sentences, then $$\lim_{n\rightarrow\infty} M^*_n(\phi_n)+M^*_n(\neg\phi_n)=1.$$
\end{lem}
\begin{proof}
  Since $\{\phi_n\}$ is quickly computable, there is a prefix $w_\phi$ such that $U(w_\phi x)$ outputs $\phi_n$ in polynomial time if $x$ encodes $n$. We can chose $w_\phi$ such that $x$ encodes $n$ in $2 \log_2(n)$ bits. Therefore, if $x$ is an infinite uniform random bitstring, $\U(w_\phi x)$ will output $\phi_n$ with probability at least $2^{-2\log_2(n)} = n^{-2}$.

Similarly, there exists a finite bitstring $w_\phi^\prime$ such that if $x$ is a uniform random bitstring, $\U(w_\phi^\prime x)$ outputs the single sentence $\neg\phi_n$ in polynomial time with probability at least $n^{-2}$. If $w_\phi$ and $w_\phi^\prime$ are each at most $k$ bits long, the probability that $w_\phi x$ and $w_\phi^\prime x$ are prefixes for some $b_i, b_j \in \{ b_1, b_2, \ldots b_{2^n} \}$ is at least $1-(1-2^{1-k} n^{-2})^{2^n}$. This converges to 1 as $n$ goes to $\infty$. Therefore, with probability converging to one, there exist $b_i$ and $b_j$ with $i,j\leq 2^n$ such that $U_{2^n}(b_i)$ is the singleton containing $\phi_n$ and $U_{2^n}(b_j)$ is the singleton containing $\neg\phi_n$. Without loss of generality, assume $i<j$. We want to show that if both of these sentences are sampled, exactly one of $\phi_n$ or $\neg \phi_n$ is in $\Phi_n$.

Clearly, the set $\{\phi_n, \neg \phi_n\}$ and all its supersets will be rejected by $\ConN{}$, so $\phi_n$ $\neg \phi_n$ are not both in $\Phi_n$. Let $\Phi_n^k$ be the value of $\Phi$ in \ClaimsN{$b_1$, $b_2$, $\ldots$} after $k$ iterations of the for loop. Assume that neither $\phi_n$ nor $\neg \phi_n$ are in $\Phi_n$. Thus, $\ConN{}$ rejects $\Phi_n^i\cup\{\phi_n\}$ and $\Phi_n^j\cup\{\neg\phi_n\}$. Adding more sentences to the input of $\ConN{}$ cannot cause it to accept, so $\ConN{}$ rejects $\Phi_n^j\cup\{\phi_n\}$ as well. Thus, $\ConN{}$ accepts $\Phi_n^i$ but rejects $\Phi_n^j\cup\{\phi_n\}$ and $\Phi_n^j\cup\{\neg\phi_n\}$. This is a contradiction, so as $n$ goes to $\infty$, with probability approaching 1, exactly one of $\phi_n$ or $\neg \phi_n$ is in $\Phi_n$.

\end{proof}

\begin{lem} \label{lem:bot0}
$\lim_{n\rightarrow\infty} M^*_n(\bot)=0.$
\end{lem}
\begin{proof}
From Lemma \ref{lem:main}, we have $\lim_{n\rightarrow\infty} M^*_n(\neg\bot)=1.$ From Lemma \ref{lem:inverse}, we have $\lim_{n\rightarrow\infty} M^*_n(\bot)+M^*_n(\neg\bot)=1.$ Together, this gives ${\lim_{n\rightarrow\infty} M^*_n(\bot)=0}.$
\end{proof}

\begin{lem} \label{lem:converges}
$\lim_{n\rightarrow\infty} M^*_n(\phi_n)$ converges whenever $\{\phi_n\}$ is quickly computable and $\phi_n\rightarrow \phi_{n+1}$ for all $n$.
\end{lem}
\begin{proof}
  Because each $\phi_n$ implies $\phi_{n+1}$, in every complete consistent extension of $\T$, either there is some greatest index $k$ such that $\phi_k$ is false, or $\phi_n$ is always false. For $k \in \{0,1,\ldots,\infty\}$, let $\{\phi_n^k\}$ be the sequence given by $\phi_n^k=\phi_n$ if $n\geq k$ and $\phi_n^k=\neg\phi_n$ otherwise. Note that for each $k$, $\{\phi_n^k\}$ is quickly computable, and that with probability~1, there exists exactly one $k$ such that $\phi_i^k\in \TotalT$ for all $i$. Note that
$$M^*_n(\phi_n)=\sum_{k}\mathbb{P}(\phi_n\in \Phi_n\mid\forall i\,\phi_i^k\in \TotalT)\mathbb{P}(\forall i\,\phi_i^k\in \TotalT),$$ so it suffices to show that $\mathbb{P}(\phi_n \in \Phi_n \mid \forall i\,\phi_i^k\in \TotalT)$ converges for each $k$, because a weighted average of bounded sequences each of which converge also converges.

Fix a $k \in \{0,1,\ldots,\infty\}.$ Let $p=\mathbb{P}(\forall i\,\phi_i^k\in \TotalT)$ and let $\varepsilon>0$. Using the same approach as in the proof of \Lem{main}, we can choose $N_1$ such that, for $n > N_1$, with probability at least $1 - p\varepsilon$, we have  $( \forall i\,\phi_i^k\in \TotalT ) \rightarrow \phi_n^k\in \Phi_n$. The probability of the conjunction $P(\phi_n^k \in \Phi_n \land \forall i\,\phi_i^k\in \TotalT)$ is then at least $p - p\varepsilon$. Therefore, ${\mathbb{P}(\phi_n^k\in \Phi_n\mid\forall i\,\phi_i^k\in \TotalT)\geq 1-\varepsilon}$ for all $n\geq N_1$, so $$\lim_{n\rightarrow\infty} \mathbb{P}(\phi_n^k\in \Phi_n\mid\forall i\,\phi_i^k\in \TotalT)$$ converges. For $k<\infty$, ${\lim_{n\rightarrow\infty} \mathbb{P}(\phi_n\in \Phi_n\mid\forall i\,\phi_i^k\in \TotalT)}$ converges, because $\{\phi_n^k\}$ is eventually just $\{\phi_n\}$.
For $k=\infty$, it also converges, because $\phi_n^k = \lnot \phi_n$, and as seen in the proof of \Lem{inverse}, with probability converging to 1, exactly one of $\phi_n$ and $\neg\phi_n$ is in $\Phi_n$.
\end{proof}

\begin{lem} \label{lem:toP}
$\lim_{n\rightarrow\infty} M^*_n(\phi_n)+M^*_n(\psi_n)+M^*_n(\chi_n)=1$ whenever $\{\phi_n\}$, $\{\psi_n\}$, and $\{\chi_n\}$ are quickly computable, and for all $n$, $\{\phi_n,\psi_n,\chi_n\}$ is a partition of truth.
\end{lem}
\begin{proof}
First, observe that $$\{(\phi_n\wedge\neg\psi_n\wedge\neg \chi_n)\vee(\neg\phi_n\wedge\psi_n\wedge\neg \chi_n)\vee(\neg\phi_n\wedge\neg\psi_n\wedge\chi_n)\}$$
is a quickly computable sequence of theorems. By \Lem{main}, each of these sentences must be in $\Phi_n$ once $n$ is large.

As seen in the proof of \Lem{inverse}, when $n$ is large we also have that with probability~1, exactly one of $\phi_n$ and $\neg\phi_n$ is in $\Phi_n$, exactly one of $\psi_n$ and $\neg\psi_n$ is in $\Phi_n$, and exactly one of $\chi_n$ and $\neg\chi_n$ is in $\Phi_n$.

Since $\Phi_n$ contains no set of sentences from which one can prove a contradiction in fewer than $2^n$ steps, this means that $\Phi_n$ must eventually contain exactly one of $\phi_n$, $\psi_n$, and $\chi_n$.
Therefore, $$\lim_{n\rightarrow\infty} \mathbb{P}(\phi_n\in \Phi_n)+ \mathbb{P}(\psi_n\in \Phi_n)+ \mathbb{P}(\chi_n\in \Phi_n)=1.$$
\end{proof}

\begin{lem}
$\lim_{n\rightarrow\infty} M^*_n(\phi)=P^*(\phi).$
\end{lem}
\begin{proof}
It suffices to show that for all $\phi$, $$\lim_{n\rightarrow\infty}\mathbb{P}(\phi\in \Phi_n)=\mathbb{P}(\phi\in \TotalT).$$

Note that in the proof of \Lem{converges}, if we take $\{\phi_n\}$ to be the sequence which is constantly $\phi$ and considering $k=0$ and $k=\infty$, we showed that for sufficiently large $n$, ${\mathbb{P}(\phi\in \Phi_n\mid\phi\in \TotalT)}\geq 1-\varepsilon$ and ${\mathbb{P}(\neg\phi\in \Phi_n\mid\neg\phi\in \TotalT)}\geq 1-\varepsilon.$ Since $\phi$ and $\neg\phi$ cannot both be in $\Phi_n$ for sufficiently large $n$, this means that $\left|\mathbb{P}(\phi\in \Phi_n)-\mathbb{P}(\phi\in \TotalT)\right| < \varepsilon$, so $\lim_{n\rightarrow\infty}\mathbb{P}(\phi\in \Phi_n)=\mathbb{P}(\phi\in \TotalT).$
\end{proof}

\section*{Acknowledgments}
This research was supported as part of the Future of Life Institute (futureoflife.org) FLI-RFP-AI1 program, grant~\#2015-144576.

\printbibliography

\end{document}